\documentclass{article}
\usepackage{amsthm}
\usepackage{amsfonts}
\usepackage{fullpage}
\usepackage{amsmath}
\usepackage[hidelinks]{hyperref}
\usepackage[ruled,vlined]{algorithm2e}
\usepackage{natbib}
\usepackage{float}
\newtheorem{theorem}{Theorem}[section]
\newtheorem{lemma}[theorem]{Lemma}

\newtheorem{assumption}[theorem]{Assumption}

\newtheorem{remark}[theorem]{Remark}

\DeclareMathOperator*{\argmin}{arg\,min}
\newcount\Comments  
\usepackage{color}
\definecolor{darkgreen}{rgb}{0,0.5,0}
\definecolor{purple}{rgb}{1,0,1}
\newcommand{\kibitz}[2]{\ifnum\Comments=1\textcolor{#1}{#2}\fi}

\title{Learning in Online MDPs: \\Is there a Price for Handling the Communicating Case?}

\author{ {\hspace{1mm}Gautam Chandrasekaran} \\
	Department of Computer Science\\
    University of Texas at Austin\\
	Austin, TX 78705, USA \\
	\texttt{gautamc@cs.utexas.edu} \\
	\and
{\hspace{1mm} Ambuj Tewari} \\
Department of Statistics\\
    University of Michigan\\
	Ann Arbor, MI 48109, USA \\
	\texttt{tewaria@umich.edu} \\}

\begin{document}
\maketitle

\begin{abstract}
It is a remarkable fact that the same $O(\sqrt{T})$ regret rate can be achieved in both the Experts Problem and the Adversarial Multi-Armed Bandit problem albeit with a worse dependence on number of actions in the latter case. In contrast, it has been shown that handling online MDPs with communicating structure and bandit information incurs $\Omega(T^{2/3})$ regret even in the case of deterministic transitions. Is this the price we pay for handling communicating structure or is it because we also have bandit feedback? In this paper we show that with full information, online MDPs can still be learned at an $O(\sqrt{T})$ rate even in the presence of communicating structure. We first show this by proposing an efficient follow the perturbed leader (FPL) algorithm for the deterministic transition case. We then extend our scope to consider stochastic transitions where we first give an inefficient $O(\sqrt{T})$-regret algorithm (with a mild additional condition on the dynamics). Then we show how to achieve $O\left(\sqrt{\frac{T}{\alpha}}\right)$ regret rate using an oracle-efficient algorithm but with the additional restriction that the starting state distribution has mass at least $\alpha$ on each state.

\end{abstract}

\section{Introduction}

In this work, we study online learning in Markov Decisions Processes. In this setting, we have an \textit{agent} interacting with an adversarial \textit{environment}. The agent observes the state of the environment and takes an action. The action incurs an associated loss and the environment moves to a new state. The state transition dynamics are assumed to be Markovian, i.e., the probability distribution of the new state is fully determined by the action and the old state. The transition dynamics are fixed and known to the learner in advance. However the losses are chosen by the adversary. The adversary is assumed to be oblivious (the entire loss sequence is chosen before the interaction begins).  We assume that the environment reveals {\it full information} about the losses at a given time step to the agent after the corresponding action is taken. The total loss incurred by the agent is the sum of losses incurred in each step of the interaction. We denote the set of states by $S$ and the set of actions by $A$. The objective of the agent is to minimized its total loss.

 This setting was first studied in the seminal work of \citet{10.2307/40538442}. They studied the restricted class of \textit{ergodic} MDPs where every policy induces a Markov chain with a single recurrent class. They designed an efficient (runs in polytime in MDP parameters and time of interaction) algorithm that achieved $O(\sqrt{T})$ regret with respect to the best stationary policy in hindsight. They assumed full information of the losses and that the MDP dynamics where known beforehand. This work was extended to {\it bandit feedback} by \citet{DBLP:journals/tac/NeuGSA14}\footnote[1]{with an additional assumption on the minimum stationary probability mass in any state}. They also achieved a regret bound of $O(\sqrt{T})$. Bandit feedback is a harder model in which the learner only receives information about its own losses.

\begin{table}[]
\begin{tabular}{l|lll}
 & \multicolumn{1}{l}{No state} & \multicolumn{1}{l}{Ergodic} & \multicolumn{1}{l}{Communicating} \\ \hline
\multicolumn{1}{l|}{Full Info} & $\sqrt{T}$ & $\sqrt{T}$ & $\sqrt{T}$\textbf{\hspace{4pt} \fbox{\tiny THIS PAPER}\vspace{2pt}} \\ \hline
\multicolumn{1}{l|}{Bandit Info} & $\sqrt{T}$ & $\sqrt{T}$\footnotemark[1] & $T^{2/3}$
 \\ \hline
\end{tabular}
\label{tab:rates}
\caption{The dependence on time horizon $T$ of the optimal regret, under full and bandit feedbacks, as the state transition dynamics become more complex.}
\end{table}

 In this paper we will look at the more general class of {\em communicating} MDPs, where, for any pair of states, there is a policy such that the time it takes to reach the second state from the first has finite expectation. In the case of bandit feedback with deterministic transitions,  \cite{10.5555/3042817.3043012} designed an algorithm that achieved $O\left(T^{2/3}\right)$ regret. This regret bound was proved to be tight by a matching lower bound in \citet{switch}. This regret lower bound was proved by a reduction from the problem of adversarial multi-armed bandits with \textit{switching costs}. In this setting, the agent incurs an additional cost every time it switches the arm it plays. Their lower bound definitively proves that in the case of bandit information, online learning over communicating MDPs is {\em statistically harder} than the adversarial multi armed bandits problem for which we have $\tilde{O}(\sqrt{T})$ regret algorithms (\citet{EXP3}). Their result gives rise to the natural question: is the high regret due to the communicating structure or bandit feedback(or both)? In the case of experts with switching cost, we know $O(\sqrt{T})$ regret algorithms such as FPL(\cite{KALAI2005291}). Using this, we give an $O(\sqrt{T})$ algorithm for online learning in communicating MDPs with full information. Thus, we show that having communicating structure alone does not add any statistical price (see Table~\ref{tab:rates}).

\subsection{Our Contributions}
In this paper, we show that online learning over MDPs with full information is not {\em statistically harder\footnote[2]{upto polynomial factors in the number of states and actions}\footnote[3]{assuming the existence of a state with a ``do nothing" action}} than the problem of online learning with expert advice. In particular, we design an efficient algorithm that learns to act in communicating ADMDPs with $O(\sqrt{T})$ regret under full information feedback against the best deterministic policy in hindsight. This is the first $O(\sqrt{T})$ regret algorithm for this problem. We prove a matching\footnotemark[2] regret lower bound in this setting. We also extend the techniques used in the previous algorithm to design an algorithm that runs in time exponential in MDP parameters that achieves $O(\sqrt{T})$ regret in the general class of communicating MDPs (albeit with an additional mild assumption\footnotemark[3]). Again, this is the first algorithm that achieves $O(\sqrt{T})$ regret against this large class of MDPs. Before this, $O(\sqrt{T})$ regret algorithms were only known for the case of ergodic MDPs. We also give an $O\left(\sqrt{\frac{T}{\alpha}}\right)$ regret algorithm for communicating MDPs with a  start state distribution having probability mass at least $\alpha$ on each state that is efficient when given access to an optimization oracle.

\section{Related Work}
The problem we study in this paper is commonly referred to in literature as online learning in MDPs over an infinite horizon. This problem was first studied for MDPs with an ergodic transition structure. \citet{10.2307/40538442} and \citet{DBLP:journals/tac/NeuGSA14} studied this problem under full information and partial information respectively. The former achieved $O(\sqrt{T})$ regret for all ergodic MDPs, whereas the latter achieved $O(\sqrt{T})$ regret for ergodic MDPs satisfying an additional assumption\footnotemark[1]. The problem of online learning in deterministic communicating MDPs(ADMDPs) was studied by \citet{10.5555/3020652.3020666} and \cite{10.5555/3042817.3043012}. They consider bandit feedback and achieve $O(T^{3/4})$ and $O(T^{2/3})$ respectively.

As mentioned in the introduction, a closely related problem is that of online learning with switching costs. In the case of full information, algorithms like FPL (\cite{KALAI2005291}) achieves $O(\sqrt{T})$ regret with switching cost. In the case of bandit feedback, \cite{switching_ub} gives an algorithm that achieves $O(T^{2/3})$ regret with switching cost. This was proved to be tight by \cite{switch} where they proved a matching lower bound.

Subsequent to the release of an earlier version of this paper, \citet{ftpl_mdp_bandit} gave an inefficient $O(T^{2/3})$ and oracle-efficient $O(T^{5/6})$ regret algorithm for online learning over Communicating\footnotemark[2] MDPs with bandit information. Their algorithms use our Switch\_Policy procedure from Algorithm~{\ref{alg:alpha}} and thus require the same assumption\footnotemark[3] as us.

\section{Preliminaries}
Fix the finite state space $S$, finite action space $A$, and transition probability matrix $P$ where $P(s,a,s')$ is the probability of moving from state $s$ to $s'$ on action $a$.

In the case of ADMDP, the transitions are deterministic and hence the ADMDP can also be represented by a directed graph $G$ with vertices corresponding to states $S$. The edges are labelled by the actions. An edge from $s$ to $s'$ labelled by action $a$ exists in the graph when the ADMDP takes the state $s$ to state $s'$ on action $a$. This edge will be referred to as $(s,a,s')$. 

A (stationary) policy $\pi$ is a mapping $\pi:S\times A\to[0,1]$ where  $\pi(s,a)$ denotes the probability of taking action $a$ when in state $s$. When the policy is deterministic, we overload the notation and define $\pi(s)$ to be the action taken when the state is $s$.
The interaction starts in an arbitrary start state is $s_1 \in S$.
The adversary chooses a sequence of loss functions $\ell_1,\ldots,\ell_T$ obliviously where each $\ell_t$ is a map from $S \times A$ to $[0,1]$.

An algorithm $\mathcal{A}$ that interacts with the online MDP chooses the action to be taken at each time step. It maintains a probability distribution over actions denoted by $\mathcal{A}(.\mid s,\ell_1,\ldots,\ell_{t-1})$ which depends on the current state and the sequence of loss functions seen so far. 
The expected loss of the algorithm $\mathcal{A}$ is 
$$L(\mathcal{A})=\mathbb{E}\left[\sum_{t=1}^{T}{\ell_{t}(s_t,a_t)}\right]$$ where $a_t\sim \mathcal{A}\left(.\mid s_t,\ell_1,\ldots,\ell_{t-1}\right), s_{t+1} \sim P(\cdot,s_t,a_t)$
For a stationary policy $\pi$, the loss of the policy is 
$$L^{\pi}=\mathbb{E}\left[\sum_{t=1}^{T}\ell_t(s_t,a_t)\right]$$ where $a_t\sim \pi(.\mid s_t), s_{t+1} \sim P(\cdot,s_t,a_t)$.
The regret of the algorithm is defined as
$$R(\mathcal{A})=L(\mathcal{A})-\min_{\pi\in \Pi}L^{\pi}\ .$$ 
The total expected loss of the best policy in hindsight is denoted by $L^*$. Thus,
$$L^*=\min_{\pi\in \Pi}L^{\pi} \ .$$

For any stationary policy $\pi$, let $T(s'\mid M,\pi,s)$ be the random variable for the first time step in which $s'$ is reached when we start at state $s$ and follow policy $\pi$ in MDP $M$.
We define the diameter $D(M)$ of the MDP as
$$D(M)=\max_{s\neq s'}\min_{\pi}\mathbb{E}\left[T(s'\mid M,\pi,s)\right].$$A {\it communicating MDP} is an MDP where $D(M) < \infty$.

\subsection{Preliminaries on ADMDPs}

In this section, we use the graph $G$ and the ADMDP interchangeably.
A stationary deterministic policy $\pi$ induces a subgraph $G_{\pi}$ of $G$ where $(s,a,s')$ is an edge in $G_{\pi}$ if and only if $\pi(s)=a$ and the action $a$ takes state $s$ to $s'$.

A communicating ADMDP corresponds to a strongly connected graph. This is because the existence of a policy that takes state $s$ to $s'$ also implies the existence of a path between the two vertices in the graph $G$. 

The subgraph $G_{\pi}$  induced by policy $\pi$ in the communicating ADMDP is the set of transitions $(s,a,s')$ that are possible under $\pi$. Each components of $G_\pi$ is either a cycle or an initial path followed by a cycle. Start a walk from any state $s$ by following the policy $\pi$. Since the set of states is finite, eventually a state must be repeated and this forms the cycle. 

Let $N(s,a)$ be the next state after visiting state $s$ and taking action $a$.
Define $I(s)$ as 
$$I(s)=\{(s',a)\mid N(s',a)=s\}.$$
The \textit{period} of a vertex $v$ in $G$ is the greatest common divisor of the lengths of all the cycles starting and ending at $v$. 
In a strongly, connected graph, the period of each vertex can be proved to be equal(\cite{markov_chains} Chap. 2, Thm 4.2). Thus, the period of a strongly connected $G$ is well defined. If the period of $G$ is 1, we call $G$ \textit{aperiodic}.

Let $\mathcal{C}_{(s,k)}$ be the set of all closed walks of $G$ of length $k$ such that the start vertex is $s$. The elements of ${\mathcal{C}_{(s,k)}}$ are represented by the sequence of edges in the walks.

Note that the cycles induced by any stationary deterministic policy $\pi$ that are of length $k$ and contain the vertex $s$ will be in $\mathcal{C}_{(s,k)}$. However, $\mathcal{C}_{(s,k)}$ can also contain cycles not induced by policies(it can contain cycles that are not simple). We use $\mathcal{C}$ to denote $\bigcup_{s\in S,k\in [k]} \mathcal{C}_{(s,k)}$.  We sometimes loosely refer elements of $\mathcal{C}$ as cycles. We define $a_t(c)$ to be the action take by $c$ in the $t_{th}$ step if we start following $c$ from the beginning of the interaction. Similarly, $s_t(c)$ is the state that you reach after following $c$ for $t-1$ steps from the start of the interaction. We define $k(c)$ as the length of the cycle $c$.

The vertices of a strongly connected graph $G$ with period $\gamma$ can be partitioned into $\gamma$ non-empty cycle classes, $C_1,\ldots,C_{\gamma}$ where each edge goes $C_{i}$ to $C_{i+1}$.

\begin{theorem}
\label{thm:critical_length}
If $G$ is strongly connected and aperiodic, there exists a critical length $d$ such that for any $\ell\geq d$, there exists a path of length $\ell$ in $G$ between any pair of vertices. Also, $d\leq n(n-1)$ where $n$ is the number of vertices in the graph. 
\end{theorem}
The above theorem is from \cite{10.2307/3689120}. It guarantees the existence of a $d>0$ such that there are paths of length $d$ between any pair of vertices. The following generalization from \cite{10.5555/3042817.3043012} extends the result to periodic graphs.
\begin{theorem}[\cite{10.5555/3042817.3043012}]
\label{thm:critical_length_gen}
If $G$ has a period $\gamma$, there exists a critical value $d$ such that for any integer $\ell\geq d$, there is a path of in $G$ of length $\gamma\ell$ from any state $v$ to any other state in the same cycle class.
\end{theorem}
\begin{remark}
We can also  find the paths of length $\ell\geq d$ from a given vertex $s$ to any other vertex $s'$ efficiently. This can be done by constructing the path in the reverse direction. We look at $P^{\ell-1}$ to see all the predecessors of $s'$ that have paths of length $\ell-1$ from $s$. We choose any of these as the penultimate vertex in the path and recurse.
\end{remark}
\section{Deterministic Transitions}
We now present our algorithm for online learning in ADMDPs when we have full information of losses. We use $G$ to refer to the graph associated to the ADMDP.

We assume that the ADMDP dynamics are known to the agent. This assumption can be relaxed as shown in \citet{ORTNER20102684} as we can figure out the dynamics in poly($|S|,|A|$) time when the transitions are deterministic.  
We want to minimize regret against the class of deterministic stationary policies. 

\subsection{Algorithm Sketch}
We formulate the task of minimizing regret against the set of deterministic policies as a problem of prediction with expert advice. As observed earlier, deterministic policies induce a subgraph which is isomorphic to a cycle with an initial path.   We keep an expert for each element of $\mathcal{C}_{(s,k)}$ for all states $s$ and $k\leq s$.  Note that we do not keep an expert for policies which have an initial path before the cycle. This is because the loss of these policies differ by at most $|S|$ compared to the loss of the cycle. Also, we make sure that the start state of the cycle is in the same cycle class as the start state of the environment. If this is not the case, our algorithm will never be \textit{in phase} with the expert policy. Henceforth, we will refer to these experts as cycles.

The loss incurred by cycle $c\in \mathcal{C}_{(s,k)}$ at time $t$ is equal to $\ell_t(s_t,a_t)$ where $s_t$ and $a_t$ are the state action pair traversed by the cycle $c$ at time $t$ if we had followed it from the start of the interaction.

We first present an efficient (running time polynomial in $|S|,|A|$ and $T$) algorithm to achieve $O(\sqrt{T})$ regret and switching cost against this class of experts. For this we used a \textit{Follow the perturbed leader} (FPL) style algorithm. 

We then use this low switching algorithm as a black box. Whenever, the black box algorithm tells us to switch policies at time $t$, we compute the state $s$ that we would have reached if we had followed the new policy from the start of the interaction and moved $t+\gamma d$ steps. We then move to this state $s$ in $\gamma d$ steps.  Theorem~\ref{thm:critical_length_gen} guarantees the existence of a path of this length. We then start following the new policy.

Thus, our algorithm matches the moves of the expert policies except when there is a switch in the policies. Thus, the regret of our algorithm differs from the regret of the black box algorithm by at most $O(\gamma d\sqrt{T})$.

\subsection{FPL algorithm}
We now describe the FPL style algorithm that competes with the set of cycles described earlier with $O(\sqrt{T})$ regret and switching cost.

\begin{algorithm}[]

        Sample perturbation vectors $\epsilon_i\in \mathbb{R}^{|S||A|}$ for $1\leq i\leq |S|$ from an exponential distribution with parameter $\lambda$\;
        Sample a perturbation vector $\delta\in \mathbb{R}^{|S|^2}$ from the same distribution\;
        \While{$t\neq T+1$}
        { 
            $C_t=\argmin_{s\in S,k\in [K],c\in \mathcal{C}_{(s,k)}} \delta(s,k)+\sum_{i=1}^{t-1}\ell_t(s_t(c),a_t(c))+\sum_{i=1}^{\max(t,k(c)+1)}\epsilon_i(s_t(c),a_t(c))$\;
        
            Adversary returns loss function $\ell_t$\;
           }
      \caption{FPL algorithm for Deterministic MDPs}
    \label{alg:fpl}
    \end{algorithm}
\subsubsection{Finding the leader: Offline Optimization Algorithm}
First, we design an offline algorithm that finds the  cycle (including start state) with lowest cumulative loss till time $t$ given the sequence of losses $\ell_1\ldots,\ell_{t-1}$. This is the $\argmin$ step in Algorithm~\ref{alg:fpl}. Given $(s,k)$, we find the best  cycle among the cycles that start in state $s$ and have length $k$. For this we use a method similar to that used in \citet{10.5555/3020652.3020666}.
We then  find the minimum over all $(s,k)$ pairs to find the best cycle. Note that we only consider start states $s$ which are in the same cycle class as the start state $s_0$ of the game.

We find the best cycle in $\mathcal{C}_{(s,k)}$ using Linear Programming.
Let $n=|S||A|k$. The LP is in the space $\mathbb{R}^n$.
Consider a cycle $c\in \mathcal{C}_{(s,k)}$. Let $c_i$ denoted the $i_{th}$ state in $c$. Also, let $a_i$ be the action taken at that state.
We associate a vector $x(c)$ with the cycle as follows.
$$x(c)_{s,a,i}=\begin{cases}
    1 & \text{if } a=a_i \text{ and } s=c_i\\
    0 & \text{otherwise}
\end{cases}$$

We construct a loss vector in $\mathbb{R}^{n}$ as follows. 
$$l_{s,a,i}=\sum_{\substack{1\leq j<t\\(j-i)\equiv 0 \mod k}}\ell_{j}(s,a)$$

Our decision set $\mathcal{X}\subseteq \mathbb{R}^n$ is the convex hull of all $x(c)$ where $c\in \mathcal{C}_{(s,k)}$. Our objective is to find $x$ in $\mathcal{X}$ such that $\langle x,l\rangle$ is minimized. The set $\mathcal{X}$ can be captured by the following polynomial sized set of linear constraints.
\begin{align*}
    &x\geq 0\\ &\sum_{a\in A}x_{(s,a,1)}=1\\
    &\forall s'\in S\setminus\{s\},a\in A,\;x_{(s,a,1)}=0\\
    &\forall (s',a')\notin I(s),\;   x_{(s',a',k)}=0\\
    &\forall s'\in S,2\leq i\leq k,\; \sum_{(s',a')\in I(s)}x_{(s',a',i-1)}=\sum_{a\in A}x_{(s,a,i)}
\end{align*}

Once we get an optimal $x$ for the above LP, we can decompose the mixed solution  as a convex combination of at most $n+1$ cycles from Caratheodory Theorem. Also, these cycles can be recovered efficiently (\citet{10.5555/1062374}). Each of them will have same loss and hence we can choose any of them.

Once we have an optimal cycle for a given $(s,k)$, we can minimize over all such pairs to get the optimal cycle. This gives us a polynomial time algorithm to get the optimal cycle.

\begin{remark}
If the new cycle chosen has the same perturbed loss as the old cycle, we will not switch. This is to prevent any unnecessary switches caused by the arbitrary choice of cycle in each optimization step (as we choose an arbitrary cycle with non-zero weight in the solution).
\end{remark}
\begin{remark}
Note that $\mathcal{C}_{(s,k)}$ can also contain cycles that don't correspond to deterministic stationary policies. However, arguing a regret upper bound against this larger class is sufficient to prove regret bounds against the class of stationary deterministic policies.
\end{remark}





\subsubsection{Regret of the FPL algorithm}
\label{sec:fpl_analysis}

We now state and prove the bound on the regret and expected number of switches of Algorithm~\ref{alg:fpl}. 
\begin{theorem}
  \label{first-order-theorem}
  The regret and the expected number of switches of Algorithm~\ref{alg:fpl} can be bounded by $$O\left(|S|\sqrt{L^*\cdot \log |S||A|}\right)$$ where $L^*$ is the cumulative loss of the best cycle in hindsight.
  \end{theorem}

Before analysing the FPL algorithm described above, we first introduce some notation and definitions. We define the loss of a cycle $c$ at time $t$ as $\ell_t(s_t(a),a_t(a))$. For any cycle $c$ with start state $s$, let $L^c$ denote the total cumulative loss that we would have received if we followed the cycle $c$ from the start to the end of the interaction. We use $\tilde{L}^c$ to denote the total perturbed cumulative loss received by cycle $c$. Let the cycle with lowest total cumulative loss be $c^*$. Also, let the cycle with lowest perturbed cumulative loss be $\tilde{c}^*$. We use $\tilde{L}^c_t$ to denote the total perturbed cumulative loss incurred by cycle $c$ after $t$ steps. We use $\tilde{c}^*_t$ to denote the cycle with lowest perturbed cumulative loss after $t$ steps. Let $C_t$ be the cycle chosen by the FPL algorithm at step $t$ and $l_t$ be it's reward. Let the expected number of switches made by the algorithm during the interaction be $N_s$.

The analysis is similar in spirit to Section 2 of \citet{KALAI2005291}. We first state the following lemma that bounds the probability of switching the cycle at any step.
\begin{lemma}  
    \label{lem:low_switching}
    $Pr[C_{t+1}\neq c\mid C_{t}=c]\leq (|S|+1)\cdot \lambda\cdot\ell_t(s_t(c),a_t(c))$ for all cycles $c$ and times $t\leq T$.
\end{lemma}

\begin{proof}[Proof of Theorem~\ref{first-order-theorem}]

We first bound the total loss incurred by the FPL algorithm.  Let the expected number of switches made by the algorithm during the interaction be $N_s$. If the algorithm doesn't switch cycles after time step $t$, then $\tilde{L}^{C_t}_t$ must be equal to $\tilde{L}^{\tilde{c}^*_t}_t$. Thus, the loss incurred at time step $t$ by $C_t$ is at most $\left(\tilde{L}^{\tilde{c}^*_t}_t-\tilde{L}^{\tilde{c}^*_{t-1}}_{t-1}\right)$. In the steps in which the algorithm switches cycles, the maximum loss incurred is $1$. Thus, we have that
\begin{align}
\label{eqn:0}
    \mathbb{E}[\text{total loss of FPL}]&\leq\tilde{L}^{\tilde{c}^*_1}+ \sum_{i=2}^{T}\left(\tilde{L}^{\tilde{c}^*_t}_t-\tilde{L}^{\tilde{c}^*_{t-1}}_{t-1}\right)+N_s\notag\\
    &\leq \tilde{L}^{\tilde{c}^*_T}+N_s\notag\\
     &=\tilde{L}^{\tilde{c}^*}+N_s
\end{align}

We now bound $N_s$. From linearity of expectation, we have that $$N_s=\sum_{t=1}^{T-1}Pr[C_{t+1}\neq C_{t}].$$ From Lemma~\ref{lem:low_switching}, we have $Pr[C_{t+1}\neq C_{t}]$ is at most $(|S|+1)\cdot \lambda\cdot \mathbb{E}[l_t]$. This gives us the following bound for $N_s$.
\begin{align*}
    N_s &=\sum_{t=1}^{T-1}Pr[C_t+1\neq C_t]\\
    &\leq  \sum_{t=1}^{T-1}(|S|+1)\cdot \lambda\cdot\mathbb{E}[l_t]\\
    &\leq (|S|+1)\cdot\lambda\cdot \sum_{t=1}^{T-1}\mathbb{E}[l_t]\\
    &\leq (|S|+1)\cdot \lambda\cdot\mathbb{E}[\text{total loss of FPL}]
\end{align*}
Combining this with \eqref{eqn:0} gives us the following.
\begin{equation}
    \label{eqn:FPL}
\mathbb{E}[\text{total loss of FPL}]\leq\tilde{L}^{\tilde{c}^*}+(|S|+1)\cdot\lambda\cdot \mathbb{E}[\text{total loss of FPL}]
\end{equation}
Let $p(c)$ denote the perturbed loss added to cycle $c$. Since the cycle with lowest perturbed cumulative loss at the end of the interaction is $\tilde{c}^*$, we have
$$\tilde{L}^{\tilde{c}^*}\leq L^{c^*}+p(\tilde{c}^*).$$ 
Also, $$\mathbb{E}[p(\tilde{c}^*)]\leq \sum_{i=1}^{|S|}\mathbb{E}\left[\max_{(s,a)}\epsilon_i(s,a)\right]+\mathbb{E}\left[\max_{(s',k)}\delta(s',k)\right]\leq |S|\cdot \frac{(1+\log{|S||A|})}{\lambda}+\frac{1+\log{|S|^2}}{\lambda}.$$ 

The above inequality comes from the fact that the expectation of the max of $k$ independant exponential random variables with parameter $\lambda$ is atmost $\frac{1+\log k}{\lambda}$. Plugging this inequality into \eqref{eqn:FPL} gives us
\begin{equation}
\label{eqn:FPL_final}
    \mathbb{E}[\text{cost of FPL}]\leq L^*+|S|\cdot \frac{(1+\log{|S||A|})}{\lambda}+\frac{1+\log{|S|^2}}{\lambda}+(|S|+1)\cdot\lambda\cdot\mathbb{E}[\text{cost of FPL}].\end{equation} Since the maximum cost is $T$, we have
$$\text{Regret}\leq |S| \frac{(1+\log{|S||A|})}{\lambda}+\frac{1+\log{|S|^2}}{\lambda}+(|S|+1)\lambda T.$$ Setting $\lambda=\frac{\log{|S||A|}}{\sqrt{T}}$ gives us a bound of $O\left(|S|\sqrt{T\log{|S||A|}}\right)$ on the regret and expected number of switches. We can also derive first order bounds. From \eqref{eqn:FPL_final}, we have  \begin{align*}
\mathbb{E}[\text{total loss of FPL}]&\leq L^*+|S|\cdot \frac{(1+\log{|S||A|})}{\lambda}+\frac{1+\log{|S|^2}}{\lambda}+(|S|+1)\cdot\lambda\cdot\mathbb{E}[\text{cost of FPL}]\\
&\leq L^*+4|S|\cdot\frac{\log |S||A|}{\lambda}+2|S|\cdot\lambda\cdot\mathbb{E}[\text{total loss of FPL}].
\end{align*}
On rearranging, we get 
\begin{align*}
  \mathbb{E}[\text{total loss of FPL}]&\leq \frac{L^*}{1-2\lambda|S|}+4|S|\cdot\frac{\log |S||A|}{\lambda(1-2\lambda|S|)}  \\
 &\leq L^*(1+(2 \lambda|S|+(2\lambda|S|)^2+\ldots)+4|S|\frac{\log |S||A|}{\lambda}(1+2\lambda|S|+(2\lambda|S|)^2+\ldots)\\
 &\leq L^*(1+4\lambda|S|)+8|S|\frac{\log|S||A|}{\lambda}.
\end{align*}
The last two inequalities work when $2\lambda|S|\leq \frac{1}{2}$. Thus,
\begin{align*}
    \mathbb{E}[\text{total loss of FPL}]-L^*\leq 4\lambda|S|(L^*)+8|S|\frac{\log |S||A|}{\lambda} .
\end{align*}
Set $\lambda=\text{min}\left(\sqrt{\frac{\log |S||A|}{L^*}},\frac{1}{4|S|}\right)$. This forces $2\lambda|S|$ to be less than $\frac{1}{2}$ and thus the previous inequalities are still valid. On substituting the value of $\lambda$, we get that 
$$\text{Regret}\leq O\left(|S|\sqrt{L^*\cdot \log |S||A|}\right)$$ when $L^*\geq 16|S|^2\log|S||A|.$
Since the expected number of switches is at most $2\lambda|S|\cdot\mathbb{E}[\text{total loss of FPL}]$, this is also bounded by $O\left(|S|\sqrt{L^*\cdot \log |S||A|}\right)$.
\end{proof}
\begin{proof}[Proof of Lemma~\ref{lem:low_switching}]
    Let $c$ be a cycle in the set $\mathcal{C}_{(s,k)}$. Let $l_t$ be shorthand for $\ell_t(s_t(c),a_t(c))$the loss incurred by cycle $c$ at step $t$. If $C_{t+1}$ is not in $\mathcal{C}_{(s,k)}$, then the algorithm must have switched. Thus, we get the following equation.
\begin{equation}
\label{eqn:1}
    Pr[C_{t+1}\neq c\mid C_{t}=c]= Pr[C_{t+1}\notin \mathcal{C}_{(s,k)}\mid C_{t}=c ]+Pr[C_{t+1}\neq {c} \text{ and }C_{t+1}\in \mathcal{C}_{(s,k)}\mid C_{t}=c]
\end{equation} 
We now bound both the terms in the right hand side of \eqref{eqn:1} separately. 

First, we study at the first term. We will upper bound this term by proving an appropriate lower bound on the probability of choosing $C_{t+1}$ from $\mathcal{C}_{(s,k)}$. Since $C_t=c$, we know that $\tilde{L}^{c}_{t-1}\leq \tilde{L}^{c'}_{t-1}$ for all $c'\neq c$. For all $c'\notin \mathcal{C}_{(s,k)}$, the perturbation $\delta(s,k)$ will play a role in the comparison of the perturbed cumulative losses. For $c'\in \mathcal{C}_{(s,k)}$, $\delta(s,k)$ appears on both sides of the comparison and thus gets cancelled out. Thus, we have $\delta(s,k)\geq w$, where $w$ depends only on the perturbations and losses received by $c$ and the cycles not in $\mathcal{C}_{(s,k)}$. Now, if $\delta(s,k)$ was larger than $w+l_t$, then the perturbed cumulative loss of $c$ will be less than that of cycles not in $\mathcal{C}_{(s,k)}$ even after receiving the losses of step $t$. In this case, $C_{t+1}$ will also be chosen from $\mathcal{C}_{(s,k)}$. This gives us the require probability lower bound.
\begin{align*}
    Pr[C_{t+1}\in \mathcal{C}_{(s,k)}\mid C_t=c]&\geq Pr[\delta(s,k)\geq w+l_t\mid \delta(s,k)\geq w]\\
    &\geq e^{-\lambda \cdot l_t}\\&\geq 1-\lambda\cdot l_t
\end{align*}
Thus, $Pr[C_{t+1}\notin \mathcal{C}_{(s,k)}\mid C_t=c]$ is at most $\lambda\cdot l_t$.

We now bound the second term. For any two cycles $c'\neq c''$ in $\mathcal{C}_{(s,k)}$, there exists an index $i\leq k$ such  that the $i^{th}$ edges of $c'$ and $c''$ are different and all the smaller indexed edges of the two cycles are the same. We denote this index by $d(c',c'')$. Define $d(c',c'')$ to be zero when  $c'$ is from $\mathcal{C}_{(s,k)}$ and $c''=c'$ or $c''$ is not from $\mathcal{C}_{(s,k)}$. Now, if $C_{t+1}$ is in $\mathcal{C}_{(s,k)}$ and not equal to $c$, then $d(C_{t+1},c)$ is a number between one and $k$. Thus, we get the following equation.

\begin{equation}
\label{eqn:2}
Pr[C_{t+1}\neq c \text{ and } C_{t+1}\in \mathcal{C}_{(s,k)}\mid C_t=c]=\sum_{i=1}^{k}Pr[d(c,C_{t+1})=i\mid C_t=c]\end{equation} 
We now bound $Pr[d(c,C_{t+1})=i\mid C_t=c]$ for any $i$ between $1$ and $k$. Let $(s_i,a_i)$ be the $i^{th}$ edge of $c$. We prove a lower bound on the probability of choosing $C_{t+1}$ such that $d(c,C_{t+1})$ is not equal to $i$. Again, since $C_t=c$, we know that $\tilde{L}^c_{t-1}\leq \tilde{L}^{c'}_{t-1}$ for all $c'\neq c$. Consider cycles $c'$ that don't contain the edge $(s_i,a_i)$ in the $i^{th}$ position. The perturbation $\epsilon_i(s_i,a_i)$ will play a role in the comparison of perturbed losses of all such $c'$ with $c$.  Thus, we have $\epsilon_i(s_i,a_i)\geq w$, where $w$ depends only on the perturbations and losses received by $c$ and  cycles $c'$ that don't have the $(s_i,a_i)$ edge in the $i^{th}$ position. If $\epsilon_i(s_i,a_i)$ was greater than $w+l_t$, then the perturbed cumulative loss of $c$ will still be less than that of all cycles $c'$ without the $(s_i,a_i)$ edge. In this case, $C_{t+1}$ will be chosen such that it also has the $(s_i,a_i)$ edge. This implies that $d(c,C_{t+1})\neq i$. Thus, we get the following probability lower bound.

\begin{align*}
    Pr[d(c,C_{t+1})\neq i\mid C_t=c]&\geq Pr[\epsilon_i(s_i,a_i)\geq w+l_t\mid \epsilon_i(s_i,a_i)\geq w]\\
    &\geq e^{-\lambda \cdot l_t}\\&\geq 1-\lambda\cdot l_t
\end{align*}
Thus, for all $i$ between $1$ and $k$, $Pr[d(c,C_t+1)=i\mid C_t=c]$ is at most $\lambda\cdot l_t$. This proves that the term in \eqref{eqn:2} is at most $k\lambda\cdot l_t$. Since $k$ is at most $|S|$, the second term in the right hand side of \eqref{eqn:1} is bounded by $|S|\cdot\lambda\cdot l_t$.
\end{proof}
\subsection{Putting it together}
We have described the FPL style algorithm
that achieves low regret and low switching. We now use Algorithm~\ref{alg:fpl} as a sub-routine to design a low regret algorithm for the online ADMDP problem. 

Recall that for a cycle $c$, $s_t(c)$ is the state you would reach if you followed the cycle $c$ from the start. This can be computed efficiently.

\begin{algorithm}[]

  t=1\;
    $s_0$ is the start state of the environment\;
     Let $c_1$ be the cycle chosen by Algorithm~\ref{alg:fpl} at $t=1$\;
    \If{$s_1\neq s_0(c_1)$}
    {
        Spend $\gamma d$ steps to move to state $s_0(c_1)$\;
        $c_{1+\gamma d}=c_1$\;
        $t=1+\gamma d$\;
    }
    
    \While{$t\neq T+1$}
    { 
       Choose action $a_t=a_t(c_t)$\;
        Adversary returns loss function $\ell_t$ and next state $s_{t+1}$\;
       Feed $\ell_t$ as the loss to Algorithm~\ref{alg:fpl} \;
         \If{Algorithm~\ref{alg:fpl} {switches cycle to } $c_{t+1}$}
        {
        \If{$s_{t+1}\neq s_{t+1}(c_{t+1})$}
    {
        Spend $\gamma d$ steps to move to state $s_{t+\gamma d}(c_{t+1})$\;
        $c_{t+\gamma d}=c_{t+1}$\;
        $t=t+\gamma d$\;

    }
           
        }
        \Else
        {
            
            $t=t+1$\;
        }
      }
  \caption{Low regret algorithm for communicating ADMDPs}
\label{alg:admdp}
\end{algorithm}

We now state the regret bound of Algorithm~\ref{alg:admdp}.
\begin{theorem}
  \label{first-order regret}
  Given a communicating ADMDP with state space $S$, action space $A$ and period $\gamma$, the regret of Algorithm~\ref{alg:admdp} is bounded by
  $$\text{Regret}\leq O\left(|S|^3\cdot \gamma\sqrt{L^*\cdot\log{|S||A|}}\right) $$ where $L^*$ is the total loss incurred by the best stationary deterministic policy in hindsight.
  \end{theorem}
  \begin{proof}
We spend $\gamma d$ steps whenever Algorithm~\ref{alg:fpl} switches. In all other steps, we receive the same loss as the cycle chosen by Algorithm~\ref{alg:fpl}. Thus, the regret differs by at most $\gamma d\cdot N_s$. From Theorem~\ref{first-order-theorem},
we get that the total regret of our algorithm in the deterministic case is  $O\left(|S|\cdot \gamma d\sqrt{T\log{|S||A|}}\right)$ where $d$ is the critical length in the ADMDP. Note that $d$ is at most $O(|S|^2)$ . Thus, we get that $$\text{Regret}\leq O\left(|S|^3\cdot \gamma\sqrt{L^*\cdot\log{|S||A|}}\right) $$ 
  \end{proof}

\begin{remark}
To achieve the first order regret bound, we set $\lambda$ in terms of $L^*$. We need prior knowledge of $L^*$ to directly do this. This can be circumvented by using a doubling trick.
\end{remark}
\subsection{Regret Lower Bound for Deterministic MDPs}
We now state a matching regret lower bound (up to polynomial factors). 
\begin{theorem}
  \label{thm:lower_bound}
For any algorithm $\mathcal{A}$ and any $|S|>3,|A|\geq1$, there exists an MDP $M$  with $|S|$ states and $|A|$ actions and a  sequence of losses $\ell_1,\ldots, \ell_t$ such that $$R(\mathcal{A})\geq \Omega\left(\sqrt{|S|T\log |A|}\right)$$ where $R(\mathcal{A})$ is the regret incurred by $\mathcal{A}$ on $M$ with the given sequence of losses.
\end{theorem}
\begin{proof}
    Let $M$ be an MDP with states labelled $s_0,s_2,\ldots, s_{|S|-1}$. Any action $a$ takes state $s_i$ to $s_{i+1}$(modulo $|S|$). In other words, the states are arranged in a cycle and every action takes any state to its next state in the cycle. This is the required $M$.

    Consider the problem of \textit{prediction with expert advice} with $n$ experts. We know that for any algorithm $\mathcal{A}$, there is a sequence of losses such that the regret of $\mathcal{A}$ is  $\Omega(\sqrt{T\log n})$ over $T$ steps (see \citet{book}). In our case, every policy spends exactly $\frac{T}{|S|}$ steps in each state. Thus, the interaction with $M$ over $T$ steps can be interpreted as a problem of prediction with expert advice at every state where each interaction lasts only $\frac{T}{|S|}$ steps. We have the following decomposition of the regret. 
    \begin{equation}\label{eqn:lb}R(\mathcal{A})=\sum_{i=0}^{|S|-1}\sum_{k=0}^{\frac{T}{|S|}-1}\ell_{k|S|+i}\left(s_i,a_{k|S|}\right)-\ell_{k|S|+i}\left(s_i,\pi^*(s_i)\right)\end{equation}
    In the above equation, $a_t$ is the action taken by $\mathcal{A}$ at step $t$. The best stationary deterministic policy in hindsight is $\pi^*$.
    
    From the regret lower bound for the experts problem, we know that there exists a sequence of losses such that for each i,  the inner sum of \eqref{eqn:lb} is atleast $\Omega\left(\sqrt{\frac{T}{|S|}\log |A|}\right)$. By combining these loss sequences, we get a sequence of losses such that 
    $$R(\mathcal{A})\geq \sum_{i=0}^{|S|-1}\Omega\left(\sqrt{\frac{T}{|S|}\log |A|}\right)\geq\Omega\left(\sqrt{|S|T\log |A|}\right).$$ This completes the proof.
    \end{proof}

\section{Stochastic Transitions}

In the previous sections, we only considered deterministic transitions. We now present an algorithm that achieves low regret for the more general class of communicating MDPs (with an additional mild restriction). This algorithm achieves $O(\sqrt{T})$ regret but takes exponential time to run (exponential in $|S|$).

\begin{assumption}
\label{asmptn:loop}
The MDP $M$ has a state $s^*$ and action $a$ such that $$Pr(s_{t+1}=s^*\mid s_t=s^*,a_t=a)=1$$
\end{assumption}
In other words, there is some state $s^*$ in which we have a deterministic action that allows us to stay in the state $s^*$. This can be interpreted as a state with a ``do nothing" action where we can wait before taking the next action.

We now state a theorem that guarantees the existence of a number $\ell^*$ such that all states can be reached from $s^*$ in exactly $\ell^*$ steps with a reasonably high probability.
\begin{theorem}
  \label{clry:critical_length}
  In MDPs satifying Assumption~\ref{asmptn:loop}, we have $\ell^*\leq 2D$ and state $s^*$ such that, for all target states $s'$, we have policies $\pi_{s'}$ such that 
  $$p_{s'}=Pr[T(s'\mid M,\pi_{s'},s^*)=\ell^*]\geq \frac{1}{4D}$$
  \end{theorem}
 We first prove an intermediate lemma.
\begin{lemma}
    \label{lem:hp_path}
    For any start state $s$ and target $s'\neq s$, we have $\ell_{s,s'}\leq 2D$ and a policy $\pi$ such that 
    $$Pr[T(s'\mid M,\pi,s)=\ell_{s,s'}]\geq \frac{1}{4D}$$
    \end{lemma}
    \begin{proof}
    From the definition of diameter, we are guaranteed a policy $\pi_{s,s'}$ such that $$\mathbb{E}\left[T(s'\mid M,\pi,s)\right]\leq D$$
    From Markov's inequality, we have 
    $$Pr\left[T(s'\mid M,\pi,s)\leq 2D\right]\geq \frac{1}{2}$$
    Since there are only $2D$ discrete values less than $2D$, there exists $\ell_{s,s'}\leq 2D$ such that $$Pr[T(s'\mid M,\pi,s)=\ell_{s,s'}]\geq\frac{1}{2}\cdot \frac{1}{2D}=\frac{1}{4D}$$
    \end{proof}
    We can now prove Theorem~\ref{clry:critical_length}
    \begin{proof}[Proof of Theorem~\ref{clry:critical_length}]
      From Lemma~\ref{lem:hp_path}, we $\ell_{s'}\leq 4D$ for each $s'$ such that there is a policy $\pi_{s^*,s'}$ that hits the state $s'$ in time $\ell_s'$ with probability at-least $\frac{1}{4D}$. We take $\ell^*=\max_{s'\neq s^*}\ell_{s'}$. For target state $s'$, the policy $\pi_{s'}$ loops at state $s^*$ for $(\ell^*-\ell_{s'})$ time steps and then starts following policy $\pi_{s,s'}$. Clearly, this policy hits state $s'$ at time $\ell^*$ with probability at least $\frac{1}{4D}$
      \end{proof}

\begin{remark}
The policies guaranteed by Theorem~\ref{clry:critical_length} are not stationary.
\end{remark}
Let $p^*=\min_{s}p_s$. Clearly, $p^*\geq \frac{1}{4D}   $
\subsection{Algorithm}
We extend the algorithm we used in the deterministic MDP case. 

We use a low switching algorithm (FPL) that considers each policy $\pi\in \Pi$ as an expert. We know from \citet{KALAI2005291} that FPL achieves $O(\sqrt{T\log{n}})$ regret as well as switching cost. At time $t$, we receive loss function $\ell_t$ from the adversary. Using this, we construct $\hat{\ell}_t$ as
$$\hat{\ell}_t(\pi)=\mathbb{E}\left[\ell_t(s_t,a_t)\right]$$
where $s_1\sim d_1,a_t\sim \pi(s_t,.)$

In other words, $\hat{\ell}_t(\pi)$ is the expected loss if we follow the policy $\pi$ from the start of the game. $d_1$ is the initial distribution of states.

We feed $\hat{\ell}_t$ as the losses to FPL. 

We can now rewrite  $L^{\pi}$ as 
$$L^{\pi}=\mathbb{E}\left[\sum_{t=1}^{T}\ell_t(s_t,a_t)\right]=\sum_{t=1}^{T}\mathbb{E}\left[\ell_t(s_t,a_t)\right]=\sum_{t=1}^{T}\hat{\ell}_t(\pi)$$ where $s_1\sim d_1$ and $a_t\sim \pi(s_t,.)$. 
Let $\pi_t$ be the policy chosen by $B$ at time $t$.
We know that
$$\mathbb{E}\left[\sum_{t=1}^{t}\hat{\ell}_t(\pi_t)\right]-\sum_{t=1}^{t}\hat{\ell}_t(\pi)\leq O(\sqrt{T\log |\Pi|})$$ for any deterministic policy $\pi$.

We need our algorithm to receive loss close to the first term in the above sum. If this is possible, we have an $O(\sqrt{T})$ regret bound for online learning in the MDP. We now present an approach to do this. 
\subsubsection{Catching a policy}
When FPL switches policy, we cannot immediately start receiving the losses of the new policy. If this was possible, then the regret of our algorithm will match that of FPL. When implementing the policy switch in our algorithm, we suffer a delay before starting to incur the losses of the new policy (in an expected sense). Our goal now is to make this delay as small as possible. This coupled with the fact that FPL has a low number of switches will give us good regret bounds. Note that this was easily done in the deterministic case using Theorem~\ref{thm:critical_length_gen}. Theorem~\ref{clry:critical_length} acts somewhat like a stochastic analogue of Theorem~\ref{thm:critical_length_gen} and we use this to reduce the time taken to catch the policy.

\begin{algorithm}

  \SetKwFunction{FMain}{Main}
  \SetKwFunction{FSwitch}{Switch\_Policy}
 \DontPrintSemicolon

  \SetKwProg{Fn}{Function}{:}{}
  \Fn{\FSwitch{$s$,$\pi$,$t_0$}}{
        $Done$ = 0\;
        $t=t_0+1$\tcp*{$t_0$ is the time that $B$ switched policy}
        $S_{t}=s$\tcp*{$S_{t}$ stores the state at time $t$}
        \While{$Done\neq 1$}
        {
            Move to state $s^{*}$ using the best policy \tcp*{Say this step takes $k$ steps}

            $t=t+k$\;
             Sample $T_{t+\ell^*}$ from $d_{\pi}^{t+\ell^*}(.)$\;
        
            We set $T_{t+\ell^*}$ as the target state\;
            Use policy $\pi_{T_{t+\ell^*}}$ guaranteed by Corollary~\ref{clry:critical_length} to move $\ell^*$ steps from $s^*$\;
            $t=t+\ell^*$\;
            \If{$S_t=T_t$}
            {
                Consider a Bernouli Random Variable $I$ such that $I=1$ with probability  $\frac{p^*}{p_{S_t}}$.\;
                \If{$I=1$}{
                Start following $\pi$ and set $Done$ to $1$\;
                Let the time at this happens be $T_{switch}$}
                \Else{$I=0$}{
                Continue\;}
            }
            \Else{
            Continue\;
            }

        }
  }
  \;

  \SetKwProg{Fn}{Function}{:}{\KwRet}
  \Fn{\FMain}{
      Let $\pi_1^\text{FPL}$  be the expert chosen by FPL at time $1$\;
      $\pi_1=\pi_1^\text{FPL}$\;
      Let $S_1$ be the start state.\;
      $t=1$\;
      \While{$t\neq T+1$}
      {
        Sample $a_t$ from $\pi_t(s_t,.)$\;
        Adversary returns loss function $\ell_t$ and next state $s$
        $S_{t+1}$=s\;
        Compute $\hat{\ell}_t$ and feed it as the loss to FPL as discussed before\;
        \If{{FPL switches policy}}
        {
            Switch\_Policy($s,\pi_{t+1}^\text{FPL},t+1$)\tcp*{Call the switch policy function to catch the new policy}
            $\pi_{t+k}=\pi_{t+1}^\text{FPL}$\tcp*{$k$ is the number of steps taking by Switch Policy}
            $t=t+k$\;
           
        }
        \Else
        {
            $\pi_{t+1}=\pi_t$\;
            $t=t+1$\;
        }
      }
  }
  
\caption{Low Regret Algorithm For Communicating MDPs}
\label{alg:alpha}
\end{algorithm}
\begin{remark}
In Algorithm~\ref{alg:alpha}, if FPL switches the policy in the middle of the Switch\_Policy's execution, we terminate the execution and call the routine again with a new target policy.
\end{remark}
\subsection{Analysis}
The following lemma shows that the Switch\_Policy routine works correctly. That is, after the execution of the routine, the state distribution is exactly the same as the state distribution of the new policy. 
\begin{lemma}
\label{lem:switch_dist}
If Switch\_Policy terminates at time $t$, we have that 
$$Pr[S_t=s\mid T_{switch}=t]=d_{\pi}^{t}(s)$$
where $d_{\pi}^{t}(s)$ is the distribution of states after following policy $\pi$ from the start of the game.
\end{lemma}
\begin{proof}[Proof of Lemma~{\ref{lem:switch_dist}}]
    We want to compute $Pr[S_t=s\mid T_{switch}=t]$.
    \begin{align*}
        Pr[S_t=s\mid T_{switch}=t]&=\frac{Pr[S_t=s,T_{switch}=t]}{Pr[T_{switch}=t]}\\
        &=\frac{Pr[S_t=T_t=s,T_{switch}=t]}{Pr[T_{switch}=t]}\\
        &=\frac{Pr[T_t=s,S_t=s,T_{switch}=t]}{Pr[T_{switch}=t]}
    \end{align*}
    We now compute the denominator $Pr[T_{switch}=t]$ as follows.
    \begin{align*}
    Pr[T_{switch}=t]&=\sum_{s\in S}Pr[S_t=T_t=s,S_{t-\ell^*}=s^*]\cdot Pr[T_{switch}=t\mid S_t=T_t=s,S_{t-\ell^*}=s^*]\\
    &=\sum_{s\in S}Pr[S_t=s\mid T_t=s,S_{t-\ell^*}=s^*]\cdot Pr[T_t=s,S_{t-\ell^*}=s^*]Pr[T_{switch}=t| S_t=T_t=s,S_{t-\ell^*}=s^*]\\
    &=\sum_{s\in S}p_s\cdot Pr[T_t=s,S_{t-\ell^*}=s^*]\cdot \frac{p^*}{p_s}\\   \
    &=p^*\sum_{s\in S}Pr[T_t=s,S_{t-\ell^*}=s^*]\\
    &=p^*\cdot Pr[S_{t-\ell^*}=s^*]
    \end{align*}
    Now we calculate the numerator.
    \begin{align*}
        Pr[T_t=s,S_t=s,T_{switch}=t]&=Pr[T_t=s,S_t=s,S_{t-\ell^*}=s,T_{switch}=t]\\
        &=Pr[S_t=s,T_{switch}=t\mid S_{t-\ell^*}=s^*,T_t=s]\cdot Pr[S_{t-\ell^*}=s^*,T_t=s]\\
        &=p^*\cdot Pr[S_{t-\ell^*}=s^*]\cdot Pr[T_t=s\mid S_{t-\ell^*}=s^*]\\
        &=p^*\cdot Pr[S_{t-\ell^*}=s^*]\cdot d_{\pi}^{t}(s)
    \end{align*}
    Thus, we have 
    $$Pr[S_t=s\mid T_{switch}=t]=d_{\pi}^{t}(s)$$
    \end{proof}

We now bound the expected loss of the algorithm in the period that FPL chooses policy $\pi$
\begin{lemma}
\label{lem:switch_cost}
Let the policy of FPL be $\pi$ from time $t_1$ to $t_2$. We have that 
$$\mathbb{E}\left[\sum_{t=t_1}^{t_2}\ell_t(s_t,a_t)\right]\leq 48\cdot D^2+\sum_{t=t_1}^{t_2}\hat{\ell}_t(\pi)$$
\end{lemma}
   
\begin{proof}[Proof of Lemma~{\ref{lem:switch_cost}}]
    We bound the expectation using law of total expectations and conditioning on $T_{switch}$.
    \begin{align*}
        \mathbb{E}\left[\sum_{t=t_1}^{t_2}\ell_t(s_t,a_t)\right]=\mathbb{E}\left[\mathbb{E}\left[\sum_{t=t_1}^{t_2}\ell_t(s_t,a_t)\mid T_{switch}\right]\right]
    \end{align*}
    We bound the conditional expectation.
    \begin{align*}
        \mathbb{E}\left[\sum_{t=t_1}^{t_2}\ell_t(s_t,a_t)\mid T_{switch}=t^*\right]\leq t^*+\mathbb{E}\left[\sum_{t=t^*}^{t_2}\ell_t(s_t,a_t)\mid T_{switch}=t^*\right]
    \end{align*}
    From Lemma~\ref{lem:switch_dist}, the second term is equal to $\sum_{t=t^*}^{t_2}\hat{\ell}_t(\pi)$
    Thus, 
    $$ \mathbb{E}\left[\sum_{t=t_1}^{t_2}\ell_t(s_t,a_t)\right]\leq \mathbb{E}[T_{switch}]+\sum_{t=t_1}^{t_2}\hat{\ell}_t(\pi)$$
    
    Everytime we try to catch the policy from state $s^*$, we succeed with probability $p^*\geq \frac{1}{4D}$. Thus, the expected number of times we try is $16\cdot D$ and each attempt takes $\ell^*\leq 2D$ steps. Between each of these attempts, we move at most $D$ steps in expectation to reach $s^*$ again. Thus, in total, we have 
    $$\mathbb{E}[T_{switch}]\leq 16D^2+32D^2=48D^2$$
    This completes the proof.
    \end{proof}
    
We are now ready to bound the regret of Algorithm~\ref{alg:alpha}
\begin{theorem}
  \label{thm:communicating_regret}
The regret of Algorithm~\ref{alg:alpha} is at most $O\left(D^2\sqrt{T\log|\Pi|}\right)$
\end{theorem}
\begin{proof}
We condition on the number of switches made by FPL. Let $N_s$ be the random variable corresponding to the number of switches made by FPL. We refer to Algorithm~\ref{alg:alpha} as $\mathcal{A}$.
\begin{align*}
    L(\mathcal{A})&=\mathbb{E}\left[\sum_{t=1}^{T}\ell_t(s_t,a_t)\right]\\
    &=\mathbb{E}\left[\mathbb{E}\left[\sum_{t=1}^{T}\ell_t(s_t,a_t)\mid N_s\right]\right]
\end{align*}

After each switch, Lemma~\ref{lem:switch_cost} tells us that the Algorithm suffers at most $48\cdot D^2$ extra average loss to the loss of the algorithm FPL. Thus, 
$$L(\mathcal{A})\leq \mathbb{E}\left[48\cdot D^2\cdot N_s+\sum_{t=1}^{T}\hat{\ell}_t(\pi_t)\right]$$
$\pi_t$ is the policy chosen by algorithm FPL at time $t$.
Since FPL is a low switching algorithm, we have ${N_s}\leq O(\sqrt{T\log |\Pi|}$. The second term in the expectation is atmost $L^\pi+O(\sqrt{T\log|\Pi|})$ for any deterministic policy $\pi$. This is because FPL is a low regret algorithm.
Thus, we have $$L(\mathcal{A})-L^\pi\leq O(D^2\sqrt{T\log|\Pi|})$$ for all stationary $\pi$.

Thus, $R(\mathcal{A})\leq O\left(D^2\sqrt{T\log|\Pi|}\right)$
\end{proof}
When $\Pi$ is the set of stationary deterministic policies, we get that $|\Pi|\leq |A|^{|S|}$. Thus, we get the following theorem.
\begin{theorem}
Given a communicating MDP satisfying Assumption~\ref{asmptn:loop} with $|S|$ states, $|A|$ action and diameter $D$, the regret of Algorithm~\ref{alg:alpha} can be bounded by
$$\text{Regret}\leq O\left(D^2\sqrt{T|S|\log |A|}\right)$$ 
\end{theorem}
In fact, since we are using FPL as the expert algorithm, we can get first-order bounds similar to Theorem~\ref{first-order-theorem}. In a setting with $n$ experts with $m$ being the total loss of the best expert, we can derive that the regret and number of switches can be bounded by $O(\sqrt{m\cdot \log n})$.Thus, using this, we get the following first order regret bounds for Algorithm~\ref{alg:alpha}
\begin{theorem}
Given a communicating MDP satisfying Assumption~\ref{asmptn:loop} with $|S|$ states, $|A|$ action and diameter $D$, the regret of Algorithm~\ref{alg:alpha} can be bounded by
$$\text{Regret}\leq O\left(D^2\sqrt{L^*\cdot |S|\log |A|}\right)$$ where $L^*$ is the total expected loss incurred by the best stationary deterministic policy in hindsight.
\end{theorem}
\subsection{Oracle-efficient algorithm assuming exploring starts}
In this section, we assume that initial distribution over states, $d_1$ has probability mass at least $\alpha$ on every state. That is, $Pr[S_1=s]\geq\alpha$ for all $s\in S$. We also assume that we have an oracle $\mathcal{O}$ that can find the stationary deterministic policy with minimum cumulative loss at no computational cost. We use this oracle and the ideas from the previous sections to design an FPL style algorithm low regret algorithm.

\subsubsection{Oracle}
The oracle $\mathcal{O}$ takes in loss functions $\ell_1,\ldots,\ell_T$ and outputs the stationary deterministic policy with the lowest expected cumulative loss. That is, it returns the policy $\pi=\argmin_{\Pi}L^{\pi}$, where $$L^{\pi}=\mathbb{E}\left[\sum_{t=1}^{T}\ell_t(s_t,a_t)\right]$$ with $s_1\sim d_1$ and $a_t=\pi(s_t)$.

We say that an algorithm is \textit{oracle-efficient} if it runs in polynomial time when given access to the oracle $\mathcal{O}$.
\subsubsection{Algorithm Sketch}
\begin{algorithm}[]

      Sample perturbation vectors $\epsilon\in \mathbb{R}^{|S||A|}$ from an exponential distribution with parameter $\lambda$\;
      \While{$t\neq T+1$}
      { 
          $$\pi_t=\argmin_{\pi\in \Pi} \mathbb{E}\left[\epsilon(s_1,a_1)+\sum_{i=1}^{t-1}\ell_t(s_t,a_t)\right]$$ with $s_1\sim d_1$ and $a_t=\pi(s_t)$\;
      
          Adversary returns loss function $\ell_t$\;
         }
    \caption{Black Box FPL algorithm used for Communicating MDPs with exploring starts}
  \label{alg:fpl_communicating}
  \end{algorithm}

Now, we use Algorithm~{\ref{alg:fpl_communicating}} as our black box experts algorithm. We prove that Algorithm~{\ref{alg:fpl_communicating}} has low regret.
\begin{theorem} \label{thm:fpl_communicating}
  The regret and expected number of switches can be bounded by $ O\left(\sqrt{\frac{L^*|S|\log |S||A|}{\alpha}}\right)$.
\end{theorem}

\begin{proof}
Let $L^{\pi}$ denote the total cumulative loss if we followed policy $\pi$ from the start of the interaction. We use $\tilde{L}^{\pi}$ to the denote the total perturbed cumulative loss if we followed policy $\pi$ from the start. Let $\pi^{*}$ be the policy with the lowest total cumulative loss. Similarly, let $\tilde{\pi}^*$ be the policy with the lowest perturbed cumulative loss. Let $\tilde{L}^{\pi}_t$ be the total perturbed cumulative loss till time $t$. Let $\pi_t$ be the policy chosen by the FPL algorithm at step $t$.

Let $N_s$ be the number of times the oracle switches the best policy. As before, we treat each policy as an expert and consider the online learning problem where expert $\pi$ gets loss $\hat{\ell}_t(\pi)=\mathbb{E}\left[\ell_t(s_t,a_t)\right]$ where $s_1\sim d_1$ and $a_t=\pi(s_t)$.

Using the arguments from the proof of Theorem~{\ref{first-order regret}}, we get $$\mathbb{E}\left[\text{total loss of FPL}\right]\leq \tilde{L}^{\tilde{\pi}^*}+N_s.$$
Also, we have $\tilde{L}^{\pi}=L^{\pi}+\sum_{s=1}^{S}d_1(s)\cdot \epsilon\left(s,\pi(s)\right)$. 

We know that $N_s=\sum_{t=1}^{T-1}Pr[\pi_{t+1}\neq \pi_t]$. We now bound $Pr[\pi_{t+1}\neq \pi_t]$. Let $\pi_t=\pi$. The algorithm chooses $\pi'\neq \pi$ as $\pi_{t+1}$ if and only if $\tilde{L}^{\pi}_t\geq \tilde{L}^{\pi'}_t$. We now argue that the probability of this happening is low if $\pi_t=\pi$. Since $\pi'\neq \pi$, we have $\pi'(s)\neq \pi(s)$ for some $s$. Let the smallest state in which $\pi$ and $\pi'$ differ be called $d(\pi,\pi')$. Thus, $$Pr[\pi_{t+1}\neq \pi\mid\pi_t=\pi]= \sum_{s\in S} Pr[d(\pi_{t+1},\pi)=s\mid \pi_t=\pi].$$ 
We bound $Pr[d(\pi,\pi_{t+1})=s\mid \pi_t=\pi]$  for any state $s$. Consider any policy $\pi'$ that differs from $\pi$ in state $s$. The perturbation $\epsilon(s,\pi(s))$ will play a role in the comparison of perturbed losses of all such $\pi'$ with $\pi$. Since $\pi_t=\pi$, we have $d_1(s)\cdot{\epsilon(s,\pi(s))}\geq w$ for some $w$ that depends only on the perturbations and losses received by $\pi$ and policies $\pi'$ that differ from $\pi$ in state $s$. If $d_1(s)\cdot{\epsilon(s,\pi(s))}\geq w+\hat{\ell}_t(\pi)$, then we would not switch to a policy $\pi'$ with $\pi(s)\neq \pi'(s)$. Thus, 
\begin{align*}
    Pr[d(\pi,\pi_{t+1})\neq s\mid \pi_t=\pi]&\geq Pr[d_1(s)\cdot\epsilon(s,\pi(s))\geq w+\hat{\ell}_t(\pi) \mid d_1(s)\cdot \epsilon(s,\pi(s))\geq w]\\
    &\geq 1-\lambda \hat{\ell}_t(\pi)\cdot\left(\frac{1}{d_1(s)}\right)
\end{align*}
Since $d_1(s)$ is at least $\alpha$, we have  $Pr[\pi_{t+1}\neq \pi\mid \pi_t=\pi]$ is at most $\frac{|S|}{\alpha} \cdot \lambda\hat{\ell}_t(\pi)$. 
From this, we get $$N_s\leq \frac{|S|}{\alpha}\cdot \lambda\cdot \mathbb{E}[\text{total loss of FPL}].$$

Using arguments similar to Section~\ref{sec:fpl_analysis}, we get 
\begin{equation}
\label{eqn:fpl_stoch}
    \mathbb{E}[\text{total loss of FPL}]\leq L^{\pi^*}+\frac{(1+\log |S||A|)}{\lambda}+\frac{|S|}{\alpha}\cdot \lambda\cdot \mathbb{E}[\text{total loss of FPL}]
\end{equation}

Let $L^*=L^{\pi^{*}}$.

On rearranging and simplifying Equation~\ref{eqn:fpl_stoch} similar to the proof of Theorem~\ref{first-order-theorem}, we have
\begin{align*}
    \mathbb{E}[\text{total loss of FPL}]\leq {L^*}\left(1+\frac{2}{\alpha}\lambda|S|\right)+4\frac{\log |S||A|}{\lambda}
\end{align*}
The above inequality works when $\frac{\lambda}{\alpha}|S|\leq \frac{1}{2}$, Thus, we have 
$$\mathbb{E}[\text{Total loss of FPL}]-L^*\leq 2\frac{\lambda}{\alpha}|S|(L^*)+4\frac{\log|S||A|}{\lambda}.$$

Set $\lambda=\min\left(\sqrt{\alpha\frac{\log|S||A|}{|S|L^*}},\frac{\alpha}{2|S|}\right)$. On substituting $\lambda$ into the above equation, we get that $$\text{Regret}\leq O\left(\sqrt{\frac{L^*|S|\log |S||A|}{\alpha}}\right).$$ Since the expected number of switches is at-most $\frac{|S|}{\alpha}\cdot \lambda\cdot \mathbb{E}[\text{total loss of FPL}]$, this is also bounded by $O\left(\sqrt{\frac{L^*|S|\log |S||A|}{\alpha}}\right)$
\end{proof}
The rest of the algorithm is the same as Algorithm~{\ref{alg:alpha}}. The exploring start assumption allows us to get an efficient low regret, low switching algorithm assuming access to the oracle $\mathcal{O}$. We now state the regret bound for this algorithm.
\begin{theorem}\label{thm:communicating_uniform_start}
Given a communicating MDP satisfying Assumption~{\ref{asmptn:loop}} with a start distribution with at least probability $\alpha$ on every state, and given access to the oracle $\mathcal{O}$, we have an efficient algorithm with 
$$Regret\leq O\left(D^2\sqrt{\frac{L^*|S|\log |S||A|}{\alpha}}\right)$$
\end{theorem}
\begin{proof}
  The proof is exactly the same as that of Theorem~{\ref{thm:communicating_regret}} except that we use the switching cost bound from Theorem~{\ref{thm:fpl_communicating}}.
\end{proof}

\section{Conclusion}
We considered learning in a communicating MDP with adversarially chosen costs in the full information setting. We gave an efficient algorithm that achieves $O(\sqrt{T})$ regret when transitions are deterministic. We also presented an inefficient algorithm that achieves a $O(\sqrt{T})$ regret bounds for the general stochastic case with an extra mild assumption. Our result show that in the full information setting there is \emph{no statistical price} (as far as the time dependence is concerned) for the extension from the vanilla online learning with experts problem to the problem of online learning with communicating MDPs.

Several interesting questions still remain open. First, what are the best lower bounds in the general (i.e., not necessarily deterministic) communicating setting? In the deterministic setting, diameter is bounded polynomially by the state space size. This is no longer true in the stochastic case. The best lower bound in terms of diameter and other relevant quantities ($|S|,|A|$ and $T$) still remains to be worked out. Second, is it possible to design an efficient algorithm beyond the deterministic case with fewer assumptions? The source of inefficiency in our algorithm is that we run FPL with each policy as an expert and perturb the losses of each policy independently. It is plausible that an FPL algorithm that perturbs losses (as in the deterministic case) can also be analyzed. However, there are challenges in its analysis as well as in proving that it is computationally efficient. For example, we are not aware of any efficient way to compute the best deterministic policy in hindsight for the general communicating case. This leads us to another open question: are there any oracle-efficient $O(\sqrt{T})$ regret algorithms that do online learning over communicating MDPs. \citet{ftpl_mdp_bandit} give an oracle-efficient $O(T^{5/6})$ regret algorithm but that works for bandits as well and does not use the additional information that is there in the full information case.

\bibliographystyle{abbrvnat}
\bibliography{references}

\begin{thebibliography}{14}
\providecommand{\natexlab}[1]{#1}
\providecommand{\url}[1]{\texttt{#1}}
\expandafter\ifx\csname urlstyle\endcsname\relax
  \providecommand{\doi}[1]{doi: #1}\else
  \providecommand{\doi}{doi: \begingroup \urlstyle{rm}\Url}\fi

\bibitem[Arora et~al.(2012{\natexlab{a}})Arora, Dekel, and
  Tewari]{10.5555/3020652.3020666}
R.~Arora, O.~Dekel, and A.~Tewari.
\newblock Deterministic mdps with adversarial rewards and bandit feedback.
\newblock In \emph{Proceedings of the Twenty-Eighth Conference on Uncertainty
  in Artificial Intelligence}, UAI'12, page 93–101, Arlington, Virginia, USA,
  2012{\natexlab{a}}. AUAI Press.
\newblock ISBN 9780974903989.

\bibitem[Arora et~al.(2012{\natexlab{b}})Arora, Dekel, and
  Tewari]{switching_ub}
R.~Arora, O.~Dekel, and A.~Tewari.
\newblock Online bandit learning against an adaptive adversary: from regret to
  policy regret.
\newblock \emph{Proceedings of the 29th International Conference on Machine
  Learning, ICML 2012}, 2, 06 2012{\natexlab{b}}.

\bibitem[Auer et~al.(1995)Auer, Cesa-Bianchi, Freund, and Schapire]{EXP3}
P.~Auer, N.~Cesa-Bianchi, Y.~Freund, and R.~Schapire.
\newblock Gambling in a rigged casino: The adversarial multi-armed bandit
  problem.
\newblock In \emph{Proceedings of IEEE 36th Annual Foundations of Computer
  Science}, pages 322--331, 1995.
\newblock \doi{10.1109/SFCS.1995.492488}.

\bibitem[Bazaraa et~al.(2004)Bazaraa, Jarvis, and Sherali]{10.5555/1062374}
M.~S. Bazaraa, J.~J. Jarvis, and H.~D. Sherali.
\newblock \emph{Linear Programming and Network Flows}.
\newblock Wiley-Interscience, USA, 2004.
\newblock ISBN 0471485993.

\bibitem[Bremaud(2000)]{markov_chains}
P.~Bremaud.
\newblock \emph{Markov Chains: Gibbs Fields, Monte Carlo Simulation, and
  Queues}.
\newblock Springer, 2000.

\bibitem[Cesa-Bianchi and Lugosi(2006)]{book}
N.~Cesa-Bianchi and G.~Lugosi.
\newblock \emph{Prediction, Learning, and Games}.
\newblock Cambridge University Press, 2006.
\newblock ISBN 978-0-521-84108-5.

\bibitem[Dai et~al.(2022)Dai, Luo, and Chen]{ftpl_mdp_bandit}
Y.~Dai, H.~Luo, and L.~Chen.
\newblock Follow-the-perturbed-leader for adversarial markov decision processes
  with bandit feedback.
\newblock \emph{arXiv preprint arXiv:2205.13451}, 2022.

\bibitem[Dekel and Hazan(2013)]{10.5555/3042817.3043012}
O.~Dekel and E.~Hazan.
\newblock Better rates for any adversarial deterministic mdp.
\newblock In \emph{Proceedings of the 30th International Conference on
  International Conference on Machine Learning - Volume 28}, ICML'13, page
  III–675–III–683. JMLR.org, 2013.

\bibitem[Dekel et~al.(2013)Dekel, Ding, Koren, and Peres]{switch}
O.~Dekel, J.~Ding, T.~Koren, and Y.~Peres.
\newblock Bandits with switching costs: $t^{2/3}$ regret.
\newblock \emph{Proceedings of the Annual ACM Symposium on Theory of
  Computing}, 10 2013.
\newblock \doi{10.1145/2591796.2591868}.

\bibitem[Denardo(1977)]{10.2307/3689120}
E.~V. Denardo.
\newblock Periods of connected networks and powers of nonnegative matrices.
\newblock \emph{Mathematics of Operations Research}, 2\penalty0 (1):\penalty0
  20--24, 1977.
\newblock ISSN 0364765X, 15265471.
\newblock URL \url{http://www.jstor.org/stable/3689120}.

\bibitem[Even-Dar et~al.(2009)Even-Dar, Kakade, and Mansour]{10.2307/40538442}
E.~Even-Dar, S.~M. Kakade, and Y.~Mansour.
\newblock Online markov decision processes.
\newblock \emph{Mathematics of Operations Research}, 34\penalty0 (3):\penalty0
  726--736, 2009.
\newblock ISSN 0364765X, 15265471.
\newblock URL \url{http://www.jstor.org/stable/40538442}.

\bibitem[Kalai and Vempala(2005)]{KALAI2005291}
A.~Kalai and S.~Vempala.
\newblock Efficient algorithms for online decision problems.
\newblock \emph{Journal of Computer and System Sciences}, 71\penalty0
  (3):\penalty0 291--307, 2005.
\newblock ISSN 0022-0000.
\newblock \doi{https://doi.org/10.1016/j.jcss.2004.10.016}.
\newblock URL
  \url{https://www.sciencedirect.com/science/article/pii/S0022000004001394}.
\newblock Learning Theory 2003.

\bibitem[Neu et~al.(2014)Neu, Gy{\"{o}}rgy, Szepesv{\'{a}}ri, and
  Antos]{DBLP:journals/tac/NeuGSA14}
G.~Neu, A.~Gy{\"{o}}rgy, C.~Szepesv{\'{a}}ri, and A.~Antos.
\newblock Online markov decision processes under bandit feedback.
\newblock \emph{{IEEE} Trans. Autom. Control.}, 59\penalty0 (3):\penalty0
  676--691, 2014.
\newblock \doi{10.1109/TAC.2013.2292137}.
\newblock URL \url{https://doi.org/10.1109/TAC.2013.2292137}.

\bibitem[Ortner(2010)]{ORTNER20102684}
R.~Ortner.
\newblock Online regret bounds for markov decision processes with deterministic
  transitions.
\newblock \emph{Theoretical Computer Science}, 411\penalty0 (29):\penalty0
  2684--2695, 2010.
\newblock ISSN 0304-3975.
\newblock \doi{https://doi.org/10.1016/j.tcs.2010.04.005}.
\newblock URL
  \url{https://www.sciencedirect.com/science/article/pii/S0304397510002008}.
\newblock Algorithmic Learning Theory (ALT 2008).

\end{thebibliography}
\newpage

%
%




%

%


\end{document}